\documentclass[11pt]{article}
\usepackage[margin=1in]{geometry}

\usepackage{cite}
\usepackage{amsmath,amssymb,amsfonts, latexsym}
\usepackage{graphicx}
\usepackage{textcomp}
\usepackage{xcolor}
\usepackage{lipsum}
\usepackage{cuted}

\usepackage{algorithm}
\usepackage{algpseudocode}
\usepackage{mathtools,nccmath}
\usepackage{subcaption}
\usepackage[bottom]{footmisc}

\let\epsilon\varepsilon

\date{\vspace{-5ex}}

\begin{document}

\newtheorem{fact}{Fact}
\newtheorem{definition}{Definition}
\newtheorem{assumption}{Assumption}
\newtheorem{lemma}{Lemma}
\newtheorem{proof}{Proof}
\renewcommand\theproof{\unskip}
\newtheorem{theorem}{Theorem}

\title{\LARGE \bf
Safe Mission Planning under Dynamical Uncertainties
\author{Yimeng Lu and Maryam Kamgarpour\thanks{This research was gratefully funded by the Swiss National Science Foundation, under the grant SNSF 200021\_172781 and the ETH Zurich Research Grant. The authors are with the Automatic Control Laboratory, Department of Information Technology and Electrical Engineering, ETH Z\"urich, Switzerland. E-mails: {\tt\footnotesize \{luyi, mkamgar\}@control.ee.ethz.ch}}}}

\maketitle

\begin{abstract}
This paper considers safe robot mission planning in uncertain dynamical environments. This problem arises in applications such as surveillance, emergency rescue, and autonomous driving. It is a challenging problem due to modeling and integrating dynamical uncertainties into a safe planning framework, and finding a solution in a computationally tractable way. In this work, we first develop a probabilistic model for dynamical uncertainties. Then, we provide a framework to generate a path that maximizes safety for complex missions by incorporating the uncertainty model. We also devise a Monte Carlo method to obtain a safe path efficiently. Finally, we evaluate the performance of our approach and compare it to potential alternatives in several case studies. 
\end{abstract}

\let\thefootnote\relax\footnote{This paper appears in International Conference of Robotics and Automation (ICRA 2020), Paris, France.}

\section{Introduction}

With the advances of robotics and artificial intelligence, autonomous robots are increasingly used in safety-critical applications. These applications, such as surveillance \cite{di2010autonomous}, emergency rescue \cite{wood2016automaton}, and autonomous driving \cite{cizelj2011probabilistically}, all require the robot to plan its path under uncertainties in the environment as well as in the model and motion of other robots. Path planning in these cases is challenging for the following reasons. First, it is required that the generated path has the highest probability of remaining safe while completing the task. Second, these uncertainties are in general dynamical and originate from different sources, e.g., uncertain locations of the obstacles \cite{zhou2018joint}, targets \cite{kamgarpour2017control}, and/or an evolving hazard \cite{wood2016automaton}. Modeling these uncertainties in a computationally tractable way and incorporating such models into a safe planning framework is a central challenge in applying robots in real-world scenarios. Last, complex missions usually involve multiple stages where several subtasks need to be fulfilled. The generated plan should handle both low-level point-to-point planning for each subtask and high-level decision making of execution order. In a search and rescue scenario, for example, the robot needs to find a path to visit all the critical locations under an evolving hazard while maximizing its safety. 
To address the challenges above, this paper aims at providing a framework for planning with safety guarantees in uncertain dynamical environments. 

In most applications of practical interest, uncertainties such as the locations of obstacles, pedestrians, vehicles, and hazardous areas, change with time. However, many of the current approaches for planning assume that these external factors are static in their initial planning stage \cite{likhachev2009probabilistic,ferguson2004pao,bajcsy2019efficient,herbert2017fastrack}. To mitigate this issue, reactive replanning methods such as D* \cite{stentz1997optimal}, D* Lite \cite{koenig2002d} and $\text{RRT}^{\text{X}}$ \cite{otte2016rrtx} adapt the initially generated path when a potential collision is detected during execution. Nonetheless, these methods are not safety maximizing as they do not account for uncertainty model during planning. Reactively adapting the path might not be globally optimal with respect to the safety of the whole mission. 
There also exist approaches to account for the uncertainty model such as \cite{gal2009efficient,phillips2011sipp}, where a known deterministic dynamic model of uncertainty is assumed for optimal safe planning. However, it is generally not realistic to assume perfect knowledge of the uncertainty dynamics in practical applications such as emergency rescue and autonomous driving.  
Chance-constrained RRT in \cite{luders2010chance} assumes a Gaussian model for dynamic obstacles and computes a probabilistically guaranteed feasible path offline. Its extension in \cite{aoude2013probabilistically} applies chance-constrained RRT to predicted future behaviors of the dynamical obstacles. Modeling the uncertainties as a partially observable Markov decision process (POMDP) is investigated in \cite{zhou2018joint} and \cite{brechtel2014probabilistic} in order to obtain a collision-free path at a crossroad with dynamic programming (DP). However, these approaches restrict uncertainties to Gaussian distributions \cite{zhou2018joint,luders2010chance,aoude2013probabilistically}, or generalize the uncertainties to Markov motion models with potentially intractable state-space \cite{brechtel2014probabilistic}. 

In scenarios with complex specifications, model-checking tools such as temporal logic verification can be used to handle both high-level decision making and low-level planning \cite{kress2011correct,lahijanian2015formal,wolff2012robust}. There also exist works in temporal logic planning with uncertainties. In \cite{maly2013iterative,guo2013revising,wongpiromsarn2010receding,sarid2013guaranteeing}, a preliminary plan is derived first and then revised reactively when the original plan is found infeasible during execution. However, similar to \cite{likhachev2009probabilistic,ferguson2004pao,bajcsy2019efficient,herbert2017fastrack}, these works do not maximize overall safety as they do not use a model of uncertainty dynamics. Incorporating uncertainty dynamics into planning frameworks increases the state space and thus, also the computational complexity \cite{wood2016automaton}. Other approaches include learning the uncertainty dynamics \cite{chen2012ltl,cizelj2011probabilistically} and incremental synthesis for temporal logic specifications \cite{ulusoy2014incremental}. However, their uncertainty models are for either specific scenarios or limited dimensions, which restricts the applicability for more general situations. 

As an example of addressing dynamical uncertainties in safety-critical applications of robotics, we consider an emergency rescue mission in an environment affected by a spreading hazard or contamination such as fire or toxic gas. We consider designing paths for the first responders or robots to maximize their safety. Using robots in safety-critical tasks is receiving increasing attention in robotics community \cite{kumar2004robot,kawatsuma2012emergency,casper2003human}. For the problem mentioned above, the uncertain hazard is dynamical, and its evolution cannot be precisely known. Developing and incorporating uncertainty models and finding computationally tractable approaches are challenging for these applications. 
Models and simulations of the hazardous environment are discussed in \cite{cheng2011dynamic,soubaras2008risk}, but these detailed models are mainly used for verification rather than stochastic control design. Recent works in emergency rescue are constrained to simple environmental settings. For instance, the works in \cite{beck2016online,baxter2007multi} discussed robot planning in earthquakes, but the models used are static. Robot-assisted evacuation is considered in \cite{schadschneider2009evacuation,shell2005insights,gorbil2011intelligent}, but dynamical uncertainty models are not incorporated. In \cite{wood2016automaton}, a similar search and rescue mission was considered for the case of firefighting in a building affected by fire, and a solution approach based on \cite{kamgarpour2017control} was introduced. While the formulation of \cite{kamgarpour2017control} allows for general environmental uncertainties, the computational approximations introduced there for tractability might result in poor safety performance.

Our contributions are summarized as follows. First, we represent the evolution of the dynamical uncertainties using a probabilistic model. Second, we develop an approximate dynamic programming algorithm to incorporate this model of uncertainties into the planning problem and solve it in a computationally tractable way; this is achieved by approximating the uncertainty evolution using a Monte Carlo method and providing corresponding safety guarantees. Third, we compare our approach with past works such as D* Lite in several case studies to show that our method provides safer path planning under dynamical uncertainties.

The rest of the paper is organized as follows. Section II formulates planning problems in uncertain dynamical environments. In Section III, we discuss our approach to handle dynamical uncertainties and solution to the planning problem in a tractable manner. In Section IV, we provide case studies for illustrating the approach and comparing it with potential alternatives. Section V concludes the paper and discusses future works. 

\section{System model and problem formulation}

We consider a general planning framework for an agent in an uncertain dynamical environment. We first define the models of the robot and uncertainties. Then, we define the planning problem for complex specifications, followed by the control synthesis of the problem.

\subsection{System model}

\subsubsection{Grid space and robot dynamics}

We use a discrete 2-dimensional grid $X=\{0,1,\ldots,m-1\} \times \{0,1,\ldots,n-1\}$ to indicate different locations in the environment. A grid cell can be free space or occupied by obstacles.
A map is a function $\mathcal{M}:X \to \{0,1\}$, whose value is 0 for free space and 1 for an obstacle such as a wall.
We use the Manhattan distance, which is the sum of the horizontal and vertical distances between cells, and assume that each grid cell has size $1\times 1$. 

We assume that an agent located at $x=(x^1,x^2) \in X$ can move in free space $X_f := \{ x \in X\ \vert \mathcal{M}(x)=0 \}$. Let $N(x)$ denote the subset of $X_f$ reachable from $x$ in one step using control input $u\in U:=\{\text{N,S,E,W,0}\}$, denoting the four directions of movement as well as staying in the current position. Given $x_t\in X_f$, the location at the next time step $x_{t+1}$ follows a distribution $x_{t+1} \sim \tau^X(\cdot \vert x_t,u_t)$, where $x_{t+1}\in N(x_t),\ u_t\in U$. 

\subsubsection{Uncertain dynamics of the environment}
To model arbitrary uncertainties with Markov dynamics, we define the parameterized stochastic set process below.

\begin{definition}[\textit{Stochastic set process}]
Let $Y := \{0,1\}^{m\times n}$ be the set of all binary matrices with the same dimension as the grid $X$. The transition probability between each $y\in Y$ is determined by a stochastic kernel $\tau^Y:Y\times Y\to \left[0,1\right]$. A stochastic set process is defined by the Markov process $y_{t+1}\sim \tau^Y(\cdot \vert y_t),~t\in\{0,1,\ldots\}$ with its initial condition $y_0$, and a set valued map $\gamma : Y \to X$. In our case, $\gamma(y)=\{x\in X | \left[y\right]_{x}=1\}$, where $\left[y\right]_{x}$ is the element at $x$ position of the binary matrix $y$.
\end{definition}

We highlight that the above stochastic set framework allows for any uncertainty with Markov dynamics. For example, parameterized geometric shapes such as ellipses or rectangles arise from uncertain vehicle motions \cite{summers2011stochastic} and probabilistic hazards. The Markov parameter $y$ can denote the center of the shape, and the map $\gamma$ can denote the volume. Note that our setup is more general as it allows for arbitrarily shaped obstacles or hazards. Consequently, it becomes more computationally challenging, and hence, it is essential to develop a method to address this problem. 
\begin{figure}
\minipage{0.48\textwidth}
    \includegraphics[width=1\linewidth]{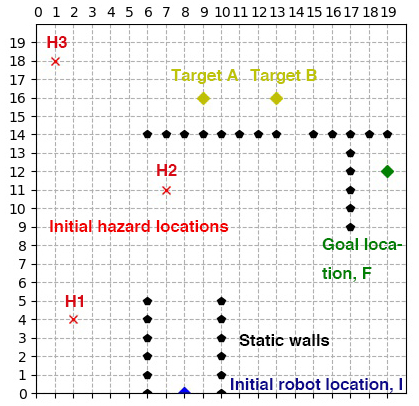}

    \caption{An example of a complex mission under uncertainties}
    \label{fig:small_example_problem}
\endminipage\hfill
\minipage{0.45\textwidth}
\includegraphics[width=1\linewidth]{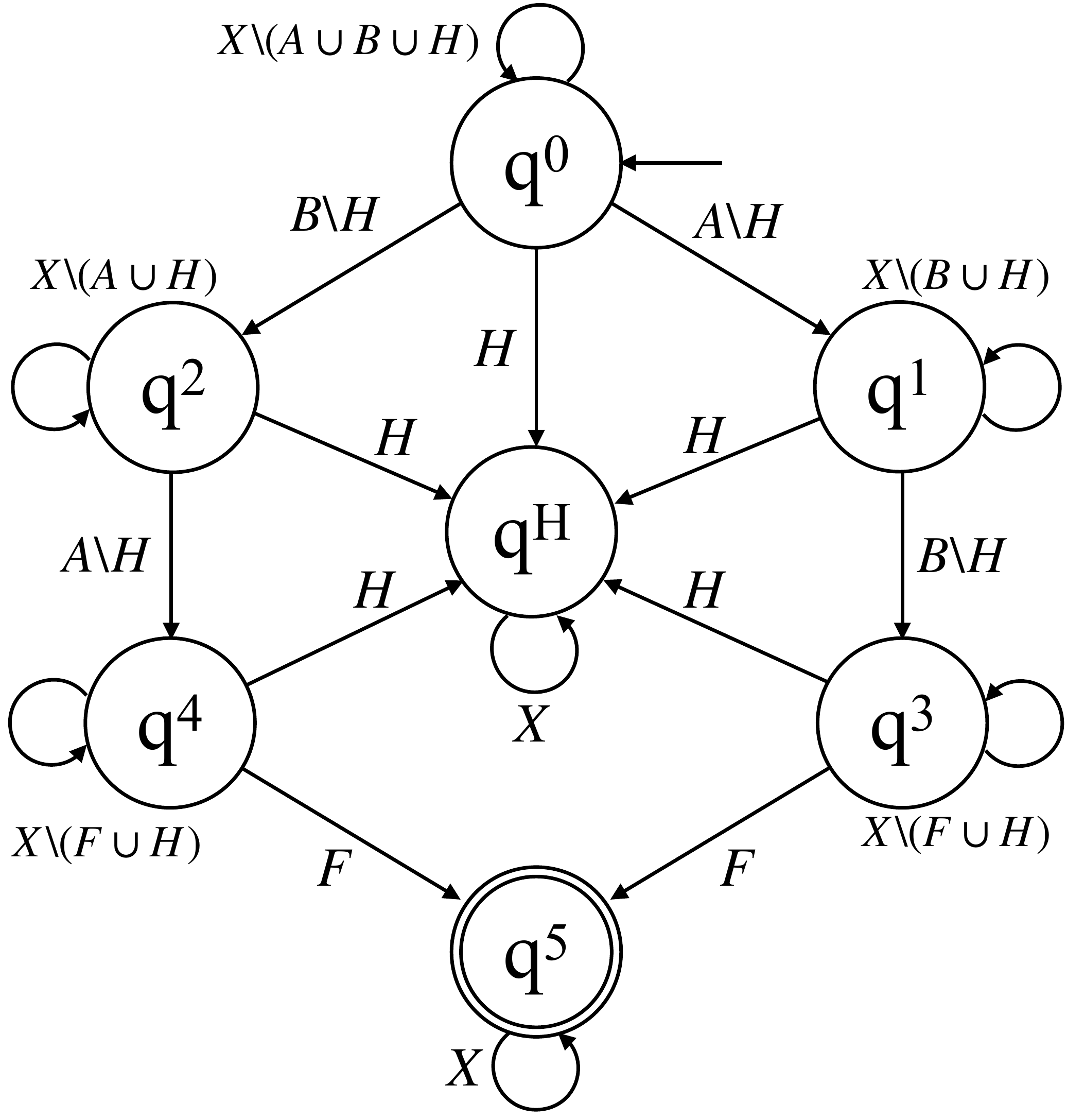}

\caption{Automaton example of the setup in Fig. 1}

\label{fig:Automaton_example}
\endminipage 
\end{figure}

\subsection{Problem formulation}
We use an evolving hazard as an example of dynamical uncertainties in the following discussions. 
The mission is to visit several targets and then reach the final goal location while keeping the robot out of the hazardous areas, which arises in search and rescue scenarios. For example, in Figure \ref{fig:small_example_problem}, the mission can be accomplished by starting from the initial location to visit targets A and B in an unspecified order, and exit at goal location F. And the hazard $H=H_1\cup H_2\cup H_3$ is shown in its initial condition. This complex mission can be described as a logical combination of visiting different targets while remaining safe and be modeled by a deterministic finite automaton (DFA) \cite{clarke2018model}. We form an overall system by combining the DFA and system dynamics. We present optimal planning as an automaton specification satisfaction problem. Note that due to the presence of the uncertainties, our objective becomes maximizing the probability of safely completing the desired task with the generated path. 
\subsubsection{Automaton specification satisfaction}
Let us formally define the finite state automaton to describe the considered specifications and the solution approach systematically. 
\begin{definition}
A finite-state automaton is a tuple $\mathcal{A}=(Q,Q^I,Q^F,\Sigma,\Delta)$, where $Q$ is a finite set of $n_q\in \mathbb{N}$ states, $Q^I\subset Q$ is a set of initial states, $Q^F\subset Q$ is a set of final (goal) states, $\Sigma$ is a finite alphabet, and $\Delta\subset Q\times \Sigma\times Q$ is a transition relation.

\end{definition}

Let the parameterized set $\gamma^{i}:Y\to X$ determine the transition to automaton state $i$. The automaton alphabet is defined as $\Sigma:=2^A$ and its transition relation is $\Delta :=\left\{(j, \sigma, i) | \sigma=\left\{K^{i}\right\}, \forall i\in Q\right\}$, where $K^{i}:= \{(x,y)\vert x\in \gamma^{i}(y)\}$ and $A:=\cup_{i=1}^{\vert Q\vert}K^{i}$ are defined in the product space $X\times Y$. For any final state $q\in Q^F$, we define $(q,X\times Y,q)\in \Delta$ making each final automaton state absorbing. This ensures once the final state is reached, the specification is satisfied as will be discussed below. We denote the hazard automaton state as $q^H$ and following the transition $(q^H,X\times Y, q^H) \in \Delta$ to make $q^H$ absorbing. This ensures once the robot enters a hazard state, the trajectory is no longer considered safe. With the automaton, we can define the specification as follows.

\begin{definition}[\textit{Specification}]\label{def:specification}
A robot trajectory $\{x_t\}_{t=0}^N$ satisfies a specification given by $\mathcal{A}$, if there exists a sequence of automaton states $(q_0,q_1,\ldots,q_N)$, such that $q_0 \in Q^I$, $(q_{t-1},(x_t,y_t),q_t)\in \Delta$ for $t=1,2,\ldots,N$, and $q_N \in Q^F$.
\end{definition}
We denote the resulting specification encoding automaton by $\mathcal{A}_s$. We make the assumption that the set $\left\{K^{i }\right\}_{i=1}^{\vert Q \vert}$ partitions $X\times Y$. Under this assumption $\mathcal{A}_s$ is deterministic and non-blocking \cite{wood2016automaton}, thus ensuring that the following transition kernel on $Q\times X\times Y$ to be well-defined. 
We denote the deterministic automaton transition kernel by $\tau^Q:Q\times Q\times X\times Y\to [0,1]$ and $\tau^Q(q'\vert q,x',y')=1_{\bar{q}'}(q')$, where $(q,(x',y'),\bar{q}')\in \Delta$.
Then, we combine the above mentioned states to form the overall system state $s \in S:= Q\times X\times Y$. The combined stochastic transition kernel between system states is $\tau\left(q', x^{\prime}, y^{\prime} | q, x, y, u\right)=\tau^{Q}\left(q' | q, x', y'\right) \tau^{X}\left(x^{\prime} | x, u\right) \tau^{Y}\left(y^{\prime} | y\right)$. This kernel provides the probability of transiting to a state $(q',x',y') \in S$ given the state $(q,x,y)\in S$ and control input $u\in U$. 
For path planning in hazardous environments, the automaton is a high-level state machine indicating the stage of the mission and the status of the system. The automaton for the mission in Figure \ref{fig:small_example_problem} is shown in Figure \ref{fig:Automaton_example}. 
\subsubsection{Control synthesis and dynamic programming solution}
Let the goal set be denoted by $G \subset S$ as $G:=Q^F \times X\times Y$. It follows that a trajectory $\{x_t\}_{t=0}^N$ satisfies a specification given by automaton if and only if $\exists t\in \{0,1,\ldots,N\}$ such that $s_t \in G$. 
As the environment is stochastic, we are concerned with the probability of a trajectory $\{x_t\}_{t=0}^N$ satisfying the specification given a series of control policy $\mathbf{\mu}:=\{\mu_0,\mu_1,\ldots,\mu_{N-1}\}$, where $\mu_t:S\to U$. This probability is defined as

\begin{equation}\label{satisfaction_probability}
r_{N}\left(s_{0}, \mathbf{\mu}\right) :=\mathbf{Pr}_{s_{0}}^{\mathbf{\mu}}\left(\bigcup_{t=0}^{N}(S \backslash G)^{t} \times G \times S^{N-t}\right), \nonumber
\end{equation}
which is maximized by an optimal $\mathbf{\mu}$. The event $(S \backslash G)^{t} \times G \times S^{N-t}$ indicates that the goal set $G$ is reached exactly at $t$, i.e., $s_t \in G \land s_k \notin G, \forall k < t$. 

A backward recursion can be derived to solve for $r^*_N(s_0)$ above, with $V^*_N(q,x,y)=1_G(q,x,y)$ and for $1\le t \le N$: 
\begin{align}\label{eq:backward_recursion_expand}
V_{t-1}^{*}(q,x,y)=& 1_{G}(q,x,y)+1_{S \backslash G}(q,x,y) \max _{u \in U} \sum_{q' \in Q} \sum_{x^{\prime}  \in X} \nonumber \\ 
 &\sum_{y^{\prime} \in Y} V_{t}^{*}\left(q',x^{\prime},y'\right)   \tau^{Q}\left(q' | q, x', y'\right) \tau^{X}\left(x^{\prime} | x, u\right) \tau^{Y}\left(y^{\prime} | y\right).
\end{align}
where the indicator function $1_{G}(s)=1$ if $s\in G$. 

It is known that that $r^*_N(s_0)=V^*_0(s_0)$ and the corresponding optimal control inputs obtained from the backward recursion give us the optimal control policies $\mu_t(s)$ for the satisfaction problem \cite{summers2010verification}. 
However, this backward recursion needs to be solved over the entire state space $S=Q\times X\times Y$, and the size of the stochastic environment state space $Y$ is exponential in the size of the grid space $X$. In the following section, we present our approach to mitigate this intractability issue by decoupling the environment state from the overall state space and exploit a Monte Carlo approximation for computational efficiency.

\section{Using safe transition probability to handle dynamical uncertainties}
In order to decouple the environment state from the backward recursion, we first propose the notion of safe transition probability to model the evolution of the hazard process using its initial condition only.  Then, to alleviate the computation complexity, we approximate the safe transition probability with a Monte Carlo approach. Using the above two techniques, we solve a computationally tractable version of (\ref{eq:backward_recursion_expand}) without $Y$ as part of the state space.

\subsection{Modeling uncertainties using safe transition probability}
To introduce our novel notion \textbf{safe transition probability (STP)}, we first recognize the redundancy and coupling in the system model from the following lemmas. 
\begin{lemma}[\textbf{Redundancy of the transition dynamics}]\label{lemma:redundancy}
The automaton transition kernel $\tau^Q(q'\vert q,x',y')$ can be determined by $q$, $x'$ and $[y']_{x'}$, i.e., only one element $[y']_{x'}$ of the binary matrix $y'$ is needed rather than the whole matrix. 
\end{lemma}
\begin{proof}
From the definition of $\tau^Q$, the value of $\tau^Q(q'\vert q,x',y')$ is 1 if $(q,(x',y'),q')\in \Delta$ and 0 otherwise. The transition from $q$ to $q'$ occurs if $x'\in \gamma^{q'}(y')$. We conclude only $[y']_{x'}$ is used in $y'$. \hfill  $\blacksquare$
\end{proof}

When computing $V_{t-1}^*(q,x,y)$ from (\ref{eq:backward_recursion_expand}), for $q=q^H$, the associated value functions are 0, i.e., $V^*_{t-1}(q^H,x,y)=0$ for $\forall t,x\ \text{and}\ y$.  
For $q\ne q^H$, $V^*_{t-1}(q,x,y)\ne 0$ is true only for the environment state $y$ such that $[y]_{x}=0$, following \textbf{Lemma 1}. This holds because in order to have non-zero value function, a state should be able to reach the goal set $G$. We conclude this with \textbf{Lemma 2}.

\begin{lemma}[\textbf{Coupling of the state elements}]\label{lemma:coupling}
Any state such that $V_t^*(q,x,y)\ne 0$ must have $q\ne q^H$ and $[y]_{x}=0$.
\end{lemma}
\begin{proof}
If $[y]_{x}=1$, we will have $q=q^H$ according to the transition dynamics. Then, $V_t^*(q,x,y)= 0$ because this state can never reach $G$ from $(q,x,y)$ at $t$. We conclude the lemma by the contraposition of this statement. \hfill $\blacksquare$
\end{proof}

In order to decouple the environment state $y$ while solving (\ref{eq:backward_recursion_expand}) at $t-1$, we need to account for the coupling between $q$ and $[y]_x$ stated in \textbf{Lemma 2}. We define the \textbf{safe transition probability (STP)} of the environment state being $y_t=y'$, with the knowledge of $y_0$, and conditioned on $q$ and $x$ as:

\begin{equation}\label{eq:markov_fire_new}
\tau_t^Y(y';q,x)=
\begin{cases}
\sum\limits_{\substack{y_{t-1} \in Y\\ [y_{t-1}]_{x} =0}} \tau^{Y}\left(y' | y_{t-1}\right) \sum\limits_{y_{t-2} \in Y} \cdots \sum\limits_{y_{1} \in Y}   \prod_{k=1}^{t-1} \tau^{Y}\left(y_{k} | y_{k-1}\right), & q\ne q^H \\
0, & q=q^H
\end{cases}
\end{equation} 
for $t\ge 1$ and $\tau_0^Y(y')=1_{y_{0}}(y')$. 

Intuitively, STP is another way of expressing probability of the environment state transiting to $y'$ at $t$, which was expressed as $\tau^Y(y'|y)$ in (\ref{eq:backward_recursion_expand}). As we will decouple the environment state, we will not be able to access $y$, but use $q$ and $x$ to infer this probability from STP.

\subsection{Incorporating STP in the backward recursion framework}
Our first result is that we can remove the environment state $y$ from the backward recursion state space using STP.  
\begin{theorem}[\textbf{Independence of hazard process and optimal solution}]
By modeling the environment state evolution using STP, the value function and the optimal control policy obtained are independent of the parameter space $Y$. 
\end{theorem}

\begin{proof}
We omit $x,y$ in $V^*_N(s)$, $1_{G}(s)$ and $1_{S\backslash G}(s)$ as they do not depend on $x\ \text{and}\ y$. 
We obtain the approximated optimal value function $\tilde{V}_{N-1}^{*}(q,x,y)$ for any state $s=(q,x,y)$ using STP as
\begin{align}
\tilde{V}_{N-1}^{*}(q,x,y)=&1_{G}(q)+1_{S \backslash G}(q) \max _{u \in U} \sum_{q' \in Q} \sum_{x^{\prime} \in X} \sum_{y^{\prime} \in Y}  \nonumber \\ & V_{N}^{*}\left(q'\right) \tau^{Q}\left(q' | q, x', y'\right) \tau^{X}\left(x^{\prime} | x, u\right) \tau_{N}^{Y}\left(y';q,x\right) \nonumber \\
=&1_{G}(q)+1_{S \backslash G}(q) \max _{u \in U} \sum_{q' \in Q} \sum_{x^{\prime} \in X}  V_{N}^{*}\left(q'\right) \tau^{X}\left(x^{\prime} | x, u\right) \nonumber \\ &  \sum_{y^{\prime} \in Y} \tau^{Q}\left(q' | q, x', y'\right) \tau_{N}^{Y}\left(y';q,x\right). \nonumber
\end{align}
In the first equality, we replace the environment dynamics $\tau^Y(y'\vert y)$ with STP at $t=N$, as STP also represents the probability of $y_N=y'$ but with the knowledge of $y_0$ only. 
From the second equality, it follows that the value function $\tilde{V}_{N-1}^*(\cdot)$ is only a function of $q$ and $x$. We use induction to conclude that $\tilde{V}^*_t(\cdot)$ is independent of the parameter $y\in Y$ for $t=N-1,\ldots,0$ by replacing $V_{t+1}^*(q',x',y')$ by $\tilde{V}_{t+1}^*(q',x')$ in the backward recursion.  \hfill $\blacksquare$
\end{proof}
We use the notation $\tilde{V}_t^*(q,x)$ in the simplified dynamic programming with STP. Note that $\tilde{V}_t^*(q,x)$ approximates $V_t(q,x,y)$ by only considering two out of three arguments, because we no longer have access to the environment state which is eliminated in the summation over $y'$. We rewrite (\ref{eq:backward_recursion_expand}) using STP as
\begin{align}\label{eq:backward_recursion_exact_separate_2}
\tilde{V}_{t-1}^{*}(q,x)= &1_{G}(q)+1_{S \backslash G}(q) \max _{u \in U} \sum_{q' \in Q} \sum_{x^{\prime} \in X}  \tilde{V}_{t}^{*}\left(q',x^{\prime}\right)\nonumber \\ & \tau^{X}\left(x^{\prime} | x, u\right) \sum_{y^{\prime} \in Y}\tau^{Q}\left(q' | q, x', y'\right)  \tau_{t}^{Y}\left(y';q,x\right). 
\end{align}
Note that when we compute $V_{t-1}^*(q,x,y)$ from (\ref{eq:backward_recursion_expand}), those states associated with $y$ such that $[y]_x=1$ have zero value functions (\textbf{Lemma 2}). Thus, $y'$ such that $[y]_x=1$ in the summation in (\ref{eq:backward_recursion_expand}) are effectively excluded for the computation of $V_{t-1}^*(q,x,y)$. When computing $\tilde{V}_{t-1}^{*}(q,x)$ in (\ref{eq:backward_recursion_exact_separate_2}), only those $y'$ that satisfy $[y]_x=0$ should be considered in the summation over $y'$ for the same reason. Also, $\tilde{V}_{t-1}^{*}(q^H,x)=0$ should be true for $\forall t$ and $x$, following the original full state solution.  
We use the following theorem that the solution of (\ref{eq:backward_recursion_exact_separate_2}) also complies with this.

\begin{theorem}\label{thm:correctness}
By using STP, the solution $\tilde{V}_{t-1}^{*}(q,x)$ obtained from (\ref{eq:backward_recursion_exact_separate_2}) satisfies: \textbf{(a)} For $q\ne q^H$, only $y'$ such that $[y]_{x}=0$ are included in the summation in backward recursion (\ref{eq:backward_recursion_exact_separate_2}). \textbf{(b)} For $q=q^H$, $\tilde{V}_{t-1}^{*}(q^H,x)=0$ for $\forall t$ and $x$.
\end{theorem}
\begin{proof}
While computing $\tilde{V}_{t-1}^{*}(q,x)$ when $q\ne q^H$, all the $y'$ in the summation satisfies $[y]_{x}=0$ by the definition of STP. For the case $q=q^H$, $\tilde{V}_{t-1}^{*}(q^H,x)=0$ is true as $\tau_{t}^{Y}\left(y';q^H,x\right)=0$ according to the definition of STP.
\hfill $\blacksquare$
\end{proof}

Based on the result above, we could perform backward recursion within a smaller state space. However, the summation $\sum_{y' \in Y} \tau^{Q}\left(q' | q, x', y'\right) \tau_{t}^{Y}\left(y';q,x\right)$ is computationally intractable. Our next contribution is to devise a tractable Monte Carlo method to compute the evolution of the environment state.

\subsection{Approximating the uncertainty evolution with Monte Carlo approach}
To simplify (\ref{eq:backward_recursion_exact_separate_2}), we define the compact transition kernel 
\begin{align}\label{eq:approximated_automaton_transition_kernel}
\tau^Q_t(q'\vert q,x',x) &:= \sum_{y' \in Y} \tau^{Q}\left(q' | q, x', y'\right) \tau_{t}^{Y}\left(y';q,x\right),
\end{align}
to indicate the probability of the simplified automaton state transiting to $q'$ from $q$ at $x'$ using STP. 
While solving (\ref{eq:backward_recursion_exact_separate_2}), we are only interested in $\tau^Q_t(q'\vert q,x',x)$ rather than the value of $\tau_{t}^{Y}\left(y';q,x\right)$ for each $y'$. 
Thus, we rewrite (\ref{eq:backward_recursion_exact_separate_2}) using the compact transition kernel as
\begin{equation}\label{eq:backward_recursion_my}
\tilde{V}_{t-1}^{*}(q,x) = 1_{G}(q) +1_{S \backslash G}(q) \max _{u \in U} \nonumber \sum_{q' \in Q} \sum_{x^{\prime} \in X}  \tilde{V}_{t}^{*}\left(q',x^{\prime}\right) \tau^{X}\left(x^{\prime} | x, u\right) \tau^{Q}_t\left(q' | q, x',x\right). 
\end{equation} 
Now our objective is to find a computationally tractable way to obtain $\tau^{Q}_t\left(q' | q, x',x\right)$. 
First, we obtain the probability of a cell being inside the hazardous area using STP. We denote the probability of $x'\in \gamma(y_t)$ for all neighbor pairs $(x,x')$ by the following quantity:

\begin{equation}\label{eq:new_covering_function_fire}
p_t(x';\gamma, q)=
\begin{cases}
\sum\limits_{y_t \in Y} 1_{\gamma(y_t)}(x') \tau_{t }^{Y}(y_t \vert [y_{t-1}]_{x}=0), & q\ne q^H \\
1, & q=q^H.
\end{cases}
\end{equation}

Then, the compact kernel $\tau_t^Q$ can be obtained as follows. First, when $q'=q^H$, the above quantity (\ref{eq:new_covering_function_fire}) is also the probability of $x'\in \gamma^{q^H}(y_t)$. Thus, $\tau^Q_t(q^H\vert q,x',x) = p_{t}(x' ; \gamma,q)$. Second, for the state where $q'\ne q^H$, the probability of $x'\in \gamma^{q'}(y_t)$ is computed as the product of the probability of $x'$ not being contaminated and the probability of the robot at a certain subset of $X_f$ determined by $q$ and $q'$. For example, in Figure \ref{fig:Automaton_example}, $\tau^Q_t(q^1\vert q^1,x',x)=(1-p_{t}(x' ; \gamma,q))(1-\mathbf{1}_{B}(x'))$. 

The quantity (\ref{eq:new_covering_function_fire}) is still costly to compute for $q\ne q^H$, but its approximation $\tilde{p}_{t}(x' ; \gamma,q)$ can be obtained by a Monte Carlo method as shown in algorithm \ref{alg:cover_2} below. 
\begin{algorithm}
\caption{\textbf{Monte Carlo approximation of $p_t(x';\gamma, q)$}}\label{alg:cover_2}
\begin{algorithmic}[1]
\State \textbf{Input:} robot state pair $(x,x')$, where $x' \in N(x)$; time $t$; total episode number $E$, hazard evolution episodes $\{y_0,y^e_1,\ldots,y^e_N\}$, where $0\le e \le E$. 
\State \textbf{Initialize:} loop counter\ $n=0$; \text{contaminated counter at}\ $t-1$, $c_{t-1}=0$; \text{contaminated counter at}\ $t$, $c_t=0$
\While{$n\le E$ }
  \State obtain $y^n_t$ and a $y^n_{t-1}$ from episode $\{y_0,y^n_1,\ldots,y^n_N\}$
  \If{$[y^n_{t-1}]_{x}=0$}
    \State $c_{t-1}=c_{t-1}+1$ 
      \If{$[y^n_t]_{x'}=1$}
        \State $c_t=c_t+1$
      \EndIf
  \EndIf\\
  $n=n+1$
\EndWhile
\State \textbf{Output:} $\tilde{p}_{t}(x' ; \gamma,q):=c_t/c_{t-1}$ for $q\ne q^H$.
\end{algorithmic}
\end{algorithm}

 Above, $\tilde{p}_{t}(x' ; \gamma,q)$ is obtained for all neighbor pairs $(x,x')\in X_f\times X_f$ at all $t\le N$. Then, we obtain the approximated (\ref{eq:approximated_automaton_transition_kernel}) as $\tilde{\tau}^Q_t(q'\vert q,x',x)$, where the $2^{\vert X\vert}$ summation due to the size of $Y$ is replaced with algorithm \ref{alg:cover_2}.

\subsection{Solving the backward recursion tractably}
By recursively solving (\ref{eq:backward_recursion_my}) with $\tilde{\tau}^{Q}_t\left(q' | q, x',x\right)$, we approximately obtain the value functions $\tilde{V}_{t-1}^{*}(q,x)$ for all the states $(q,x)$ and for $t=N,\ldots,0$.

To summarize the approaches introduced so far, we firstly simplified (\ref{eq:backward_recursion_expand}) to (\ref{eq:backward_recursion_exact_separate_2}) by decoupling the hazard process $\tau^Y_t(y';q,x)$ using STP. Then, we approximated the combined kernel $\tau^Q_t(q'\vert q,x',x)$ using $\tilde{p}_{t}(x' ; \gamma,q)$ from Monte Carlo simulations to solve (\ref{eq:backward_recursion_exact_separate_2})  in a tractable way as (\ref{eq:backward_recursion_my}). 
Note that we lose accuracy in the proposed approach to alleviate the computation complexity during two processes. One is the decoupling of the environment state from the backward recursion. Although we use STP to capture the coupling relation as in Theorem \ref{thm:correctness}, we still cannot reach the exact result as solving (\ref{eq:backward_recursion_expand}). The other one is the Monte Carlo approximation of the hazard process. This inaccuracy can be mitigated by using a large number of Monte Carlo samples in algorithm \ref{alg:cover_2} as it is an offline computation. 

\section{Case study}

We present case studies of planning in an uncertain environment to compare the performance of our method with the most capable existing approaches to address dynamical uncertainties as far as we are concerned. We show that our method provides trajectories with a higher probability of safety than D* Lite and the method in \cite{wood2016automaton} for point-to-point and complex mission planning, respectively. 

We use an evolving fire based on \cite{soubaras2008risk} as an example of the dynamical uncertainties in the following experiments. Starting from an initial location, the fire spreads to the neighboring grid cells probabilistically. Particularly, the transition kernel of the set process $\tau^Y:Y\times Y\to[0,1]$ is
\begin{equation}\label{eq:hazard_transition_kernel}
\tau^{Y}\left(y_{t+1} | y_t\right)=\prod_{x^1=0}^{m-1} \prod_{x^2=0}^{n-1} \mathbf{Pr}_{f}\Big\{\left[y_{t+1}\right]_{x} | y_t\Big\},
\end{equation}

\begin{equation}\label{eq:hazard_transition}
\small
\mathbf{Pr}_{f}\Big\{[y_{t+1}]_{x}\vert y_t\Big\}
:=
\begin{cases} 
\displaystyle
p_{n}(x, y_t) & \text{if}\ [y_{t+1}]_{x}=0\wedge[y_t]_{x}=0\\ 
1-p_{n}(x, y_t) & \text{if}\ [y_{t+1}]_{x}=1\wedge[y_t]_{x}=0\\ 
0 & \text{if}\ [y_{t+1}]_{x}=0\wedge[y_t]_{x}=1\\
1 & \text{if}\ [y_{t+1}]_{x}=1\wedge[y_t]_{x}=1, 
\end{cases}
\end{equation}
where $p_n(x,y_t)$ denotes the probability of a safe grid cell $x$ remaining safe in the next time step, and is defined as: 
\begin{equation}
\label{eq:hazard_safeProb}
p_n(x,y_t) = \Big( 1 - p_f(x) \Big)^{N_f(x,y_t)}\Big( 1 - \frac{p_f(x)}{\sqrt{2}} \Big)^{D_f(x,y_t)}.
\end{equation}
For each grid $x$, $p_f(x)$ is a constant describing the evolving speed of the hazard in different types of the space.
Parameters $N_f$ and $D_f$ in (\ref{eq:hazard_safeProb}) are the numbers of direct and diagonal neighbors inside the hazardous areas. For diagonal neighbors, we scale the parameter $p_f(x)$ by $\frac{1}{\sqrt{2}}$, which indicates the flexibility to set different hazard evolving speed based on the distance of the neighboring cell. 
The hazard model described by (\ref{eq:hazard_transition_kernel}), (\ref{eq:hazard_transition}) and (\ref{eq:hazard_safeProb}) can be intuitively explained as follows. 
From (\ref{eq:hazard_transition}), a contaminated cell remains contaminated in all following time steps. The probability of an uncontaminated cell getting contaminated, i.e., affected by fire, in the next time step is independently affected by the environment states of its neighbors at the current time step, as shown in (\ref{eq:hazard_safeProb}). Using this process, we obtain $[y_{t+1}]_x$, for $\forall x$. Then, the probability of any environment state $y_{t+1}$ can be computed by multiplying the probability of each grid cell $x$ holding the value of $[y_{t+1}]_x$ given $y_t$ as in (\ref{eq:hazard_transition_kernel}).

\subsection{Comparison to D* Lite in point-to-point planning}
We set up a point-to-point planning case by simplifying the complex mission we use in Figure \ref{fig:small_example_problem}. We only kept one target $A$ and three initial hazards $H_1$, $H_2$ and $H_3$. The objective is to plan a path from the initial location $I$ to target $A$ while maximizing the probability of finishing this mission safely. The setup is shown in Figure \ref{fig:ptp}(a). As discussed in the introduction, D* Lite is a reactive replanning algorithm to find the shortest path in dynamical environments. The initially planned path is adjusted when an obstacle is detected within the robot's visibility during the execution of the path. We set the visibility of the robot to be its neighbor grids within $2$ steps. For our method, we used the offline closed-loop policy obtained from (\ref{eq:backward_recursion_my}) and did not adjust it during execution. Note we obtain optimal control inputs for any given state at any time accounting for the dynamical uncertainties, while D* Lite only considers a static environment. In Figure \ref{fig:ptp}(a), the initial D* Lite path (the green line) directly went to the target as it does not consider the dynamical model of the environment, while the path planned by our method (the blue line) took a detour. We tested the two methods by simulating the hazard spreading process 1000 times to compare their empirical probabilities of safely reaching the target, as shown in the table.

\begin{center}
 \begin{tabular}{||c c c||} 
 \hline
 Method &D* Lite&  This work  \\ [0.5ex] 
 \hline
 Success rate & $30.0\%$ & $38.7\%$ \\ 
 \hline
\end{tabular}
\end{center}

For D* Lite, the robot sometimes backtracked its path when blocked by hazard during the execution. One example is shown in Figure \ref{fig:ptp}(b) and Figure \ref{fig:ptp}(c), where the path already traveled is solid, and the path planned but not yet executed is dotted. At $t=18$, the robot was at $(12,10)$ and planned to go to $(13,10)$ at the next time step as in Figure \ref{fig:ptp}(b). But at $t=19$, the robot observed that $(13,10)$ was blocked by hazard and it had to go back to $(12,9)$ and plan a new path as in Figure \ref{fig:ptp}(c). At $t=19$, the robot had already reached $(13,14)$ using our method as in Figure \ref{fig:ptp}(d). Note that this might happen several times, like in the example, and the robot using D* Lite will have a much longer detour. Then, the hazard could contaminate more spaces and decrease the probability of finishing the task safely. 
\begin{figure}
  \begin{subfigure}[t]{.48\textwidth}
    \centering
    \includegraphics[width=\linewidth]{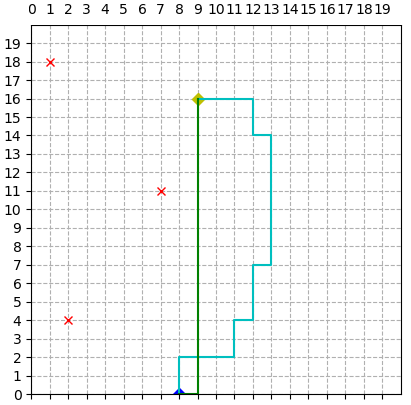}
    \caption{Initial path by both methods}
  \end{subfigure}
  \hfill
  \begin{subfigure}[t]{.48\textwidth}
    \centering
    \includegraphics[width=\linewidth]{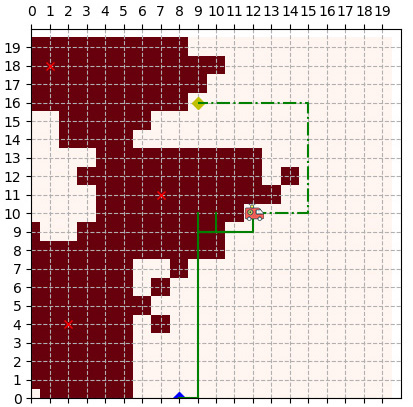}
    \caption{D* Lite at $t=18$}
  \end{subfigure}

  \medskip

  \begin{subfigure}[t]{.48\textwidth}
    \centering
    \includegraphics[width=\linewidth]{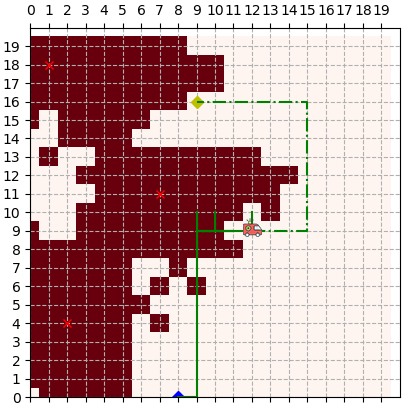}
    \caption{D* Lite at $t=19$}
  \end{subfigure}
  \hfill
  \begin{subfigure}[t]{.48\textwidth}
    \centering
    \includegraphics[width=\linewidth]{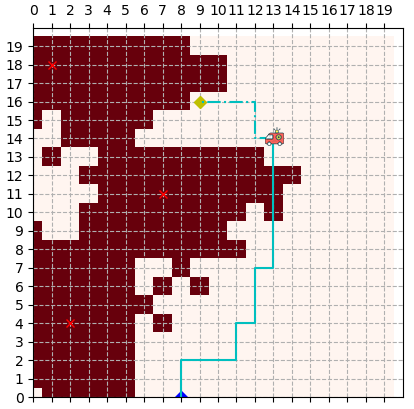}
    \caption{Our method at $t=19$}
  \end{subfigure}
  \caption{The paths from D* Lite and our method are shown in green lines and blue lines, respectively. Path already traveled is solid and not yet executed is dotted. Because D* Lite backtracked and replanned several times as in (b) and (c), our method yielded better path than D* Lite under dynamical uncertainties as in (d).}
  \label{fig:ptp}
\end{figure}

\subsection{Comparison to past works in complex mission planning}
We compare our method to \cite{wood2016automaton}, where a similar DP is developed but the coupling between state elements is ignored.
We show that we obtain a better path using our method in terms of the probability of satisfying the specification safely than the method in \cite{wood2016automaton}. In this example, we used the specification of reaching the target $A$ first and exiting at $F$, while avoiding the hazard $H$. Some obstacles were added to increase the complexity of the task. The time consumption to solve the DP for our method and the method in \cite{wood2016automaton} were both approximated 150 seconds. The paths obtained offline using both methods are shown in Figure \ref{fig:exp}, where our method chose a longer but safer path. Note that red color indicates that the grid cell was contaminated at a certain time, where darker red means the cell was contaminated earlier. We show the empirical success rates for both methods in the following table. 

\begin{figure}
  \begin{subfigure}[t]{.48\textwidth}
    \centering
    \includegraphics[width=\linewidth]{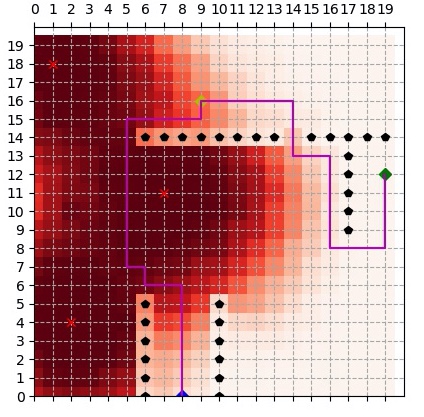}
  \end{subfigure}
  \hfill
  \begin{subfigure}[t]{.48\textwidth}
    \centering
    \includegraphics[width=\linewidth]{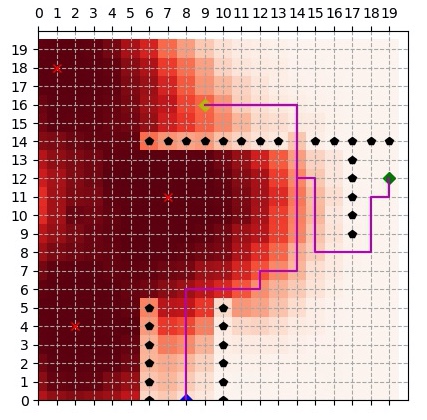}
  \end{subfigure}
  \caption{Optimal path from \cite{wood2016automaton} (left) and our method (right). In this case our method planed a longer but safer path, which outperformed the method in \cite{wood2016automaton}.}
  \label{fig:exp}
\end{figure}

\begin{center}
 \begin{tabular}{||c c c||} 
 \hline
 Method &DP without STP in \cite{wood2016automaton}&  This work   \\ [0.5ex] 
 \hline
 Success rate & $0.003\%$ & $0.667\%$\\ 
 \hline
\end{tabular}
\end{center}

Note that for the case studies, due to the size of the environment state space, we could not obtain a ground truth for the planning problem, and thus had to use the approximated approaches. The algorithms were coded in Python 3.6 on a computer with 2.3 GHz i5 quadcore processor and 16 GB memory. 

\section{Conclusions}
We developed a framework of safe planning for complex missions in uncertain dynamical environments. To address computational tractability, we devised a method to decouple the computation of the environment's dynamical evolution from that of the robot trajectory to solve the planning problem, based on a novel notion of safe transition probability. Also, we mitigated the intractability issue by designing a Monte Carlo approximation for the environment's uncertainty propagation. We showed in case studies that our method outperforms existing approaches for planning under dynamical uncertainties. 

Future directions include combining the idea of replanning approach employed in such as D* Lite in our framework. This enables incorporating online observations and making corresponding online adjustments to improve the safety probability further.
Furthermore, we plan to run the methodology on realistic robotics testbeds.

\section*{Acknowledgement}
The authors thank Dr. David Adjiashvili for helpful discussions. 

\bibliographystyle{plain}
\bibliography{references}

\end{document}